\documentclass{article}

\usepackage[preprint]{neurips_2024}

\usepackage[utf8]{inputenc} 
\usepackage[T1]{fontenc}    
\usepackage{hyperref}       
\usepackage{url}            
\usepackage{booktabs}       
\usepackage{amsfonts}       
\usepackage{nicefrac}       
\usepackage{microtype}      
\usepackage{xcolor}         

\usepackage{amsmath}
\usepackage{amssymb}
\usepackage{mathtools}
\usepackage{amsthm}
\theoremstyle{plain}
\newtheorem{theorem}{Theorem}[section]
\newtheorem{proposition}[theorem]{Proposition}

\theoremstyle{definition}
\newtheorem{definition}[theorem]{Definition}

\theoremstyle{remark}
\newtheorem{remark}[theorem]{Remark}
\newtheorem{proposition*}[theorem]{Proposition}
\newtheorem{theorem*}[theorem]{Theorem}

\usepackage{siunitx}

\usepackage{multirow}

\usepackage[capitalize,noabbrev]{cleveref}
\usepackage{subcaption}
\usepackage{tikz}
\usetikzlibrary{calc,arrows.meta}
\usepackage{bm}
\title{Simplifying the Theory on Over-Smoothing}

\author{%
  Andreas Roth \\
  TU Dortmund University\\
  44227 Dortmund, Germany \\
  \texttt{andreas.roth@tu-dortmund.de} \\
}

\begin{document}

\maketitle

\begin{abstract}
Graph convolutions have gained popularity due to their ability to efficiently operate on data with an irregular geometric structure. However, graph convolutions cause over-smoothing, which refers to representations becoming more similar with increased depth.
However, many different definitions and intuitions currently coexist, leading to research efforts focusing on incompatible directions. This paper attempts to align these directions by showing that over-smoothing is merely a special case of power iteration. This greatly simplifies the existing theory on over-smoothing, making it more accessible. Based on the theory, we provide a novel comprehensive definition of rank collapse as a generalized form of over-smoothing and introduce the rank-one distance as a corresponding metric.
Our empirical evaluation of 14 commonly used methods shows that more models than were previously known suffer from this issue.
\end{abstract}

\section{Introduction}
When operating with message-passing neural networks on graph-structured data, over-smoothing describes a phenomenon in which node representations become more similar when the number of convolutional layers increases. 
Many research efforts provide theoretical insights on over-smoothing and methods to mitigate its effects~\citep{zhou2020towards,Zhao2020PairNorm,Rong2020DropEdge,chen2020measuring,rusch2022graph,roth2022transforming,rusch2023gradient}.
However, due to the multitude of different theoretical insights and their complexity, different research efforts often use distinct definitions for over-smoothing, which are partly incompatible. In particular, some works study normalized representations~\citep{digiovanni2023understanding,roth2023rank,NEURIPS2023_2a514213} while others consider unnormalized representations~\citep{NEURIPS2023_6e4cdfdd,scholkemper2024residual,rusch2023survey}. Some define over-smoothing as the convergence to a constant state~\citep{rusch2022graph,rusch2023gradient,NEURIPS2023_6e4cdfdd,rusch2023survey,scholkemper2024residual}, others claim different limit distributions depending on the spectrum of the aggregation function~\citep{li2018deeper,cai2020note,Oono2020Graph,zhou2021dirichlet,digiovanni2023understanding,roth2023rank,NEURIPS2023_2a514213}.

To combine these strands, we show that the theory behind over-smoothing can be greatly simplified and reduced by connecting it to the classical power iteration method~\citep{kowalewski1909einfuhrung,miintz1913solution,mises1929praktische}. While our resulting insights are not novel, our novel proofs aim to make the theory more accessible to a broader part of the community. We first recap power iteration with its in-depth proof. We show that most graph convolutions represent a particular case for which the dominant eigenvector is a Kronecker product. Its properties lead to over-smoothing, for which we provide a novel theoretically founded definition. 

\section{Preliminaries}
\label{sec:related}
\paragraph{Notation} Let $G=(\mathcal{V},\mathcal{E})$ be a graph consisting of a node set $\mathcal{V}=\{v_1,\dots,v_n\}$ and an edge set $\mathcal{E}$. The matrix representations of $G$ is given by its adjacency matrix $\mathbf{A}\in\{0,1\}^{n\times n}$, for which $\mathbf{A}_ij = 1$ only if an edge between nodes $v_i$ and $v_j$ exists. For a given node $v_i$, the set of neighboring nodes is given by $\mathcal{N}_i = \{ v_j \mid (j,i)\in\mathcal{E}\}$. The node degree is given by $d_{ii}=|\mathcal{N}_i|$ and the corresponding degree matrix $\mathbf{D}\in\mathbb{N}^{n\times n}$ with the node degrees along its diagonal. 
The eigenvalues of a matrix $\mathbf{M}$ are denoted by $\lambda_1^\mathbf{M},\dots,\lambda_n^{\mathbf{M}}$ and sorted with decreasing magnitude.
The vectorize operation $\mathrm{vec}(\mathbf{M})$ is defined as stacking the columns of $\mathbf{M}$ into a single vector.

\paragraph{Graph neural networks} Given a graph $G$ and a $d$-dimensional node signal $\mathbf{X}\in\mathbb{R}^{n\times d}$, we want to obtain expressive node embeddings capturing information given by both the data structure and the signals. These embeddings are utilized in downstream tasks like node classification and graph classification. Most graph convolutions apply a node mixing function $\mathbf{\Tilde{A}}\in\mathbb{R}^{n\times n}$ that represents the graph structure and its edge weights, and a feature transformation $\mathbf{W}\in\mathbb{R}^{d\times d}$. In matrix notation, these can be expressed as iterative transformations of the form
\begin{equation}
\label{eq:update}
    \mathbf{X}^{(k+1)} = \Tilde{\mathbf{A}}\mathbf{X}^{(k)}\mathbf{W}^{(k)}\, ,
\end{equation}
Popular instantiations covered by this notation include the graph convolutional network (GCN)~\citep{kipf2017semisupervised}, the graph isomorphism network (GIN)~\citep{xu2018how}, and the graph attention network (GAT)~\citep{veličković2018graph}.

\paragraph{Over-smoothing in GNNs}
Over-smoothing describes a phenomenon in which the node representations become excessively similar as the number of layers $k$ increases. As many definitions and intuitions for over-smoothing coexist and in some cases contradict each other, we want to familiarize the reader with the different studies.

\citet{li2018deeper} has shown that the normalized adjacency matrix employed by the GCN performs a special form of Laplacian smoothing, leading to over-smoothing when many iterations are performed.
However, their study did not consider the role of the feature transformation. 
\citet{Oono2020Graph} studied the distance of $\mathbf{X}^{(k)}$ to a subspace $\mathbf{M}$ that is spanned by the dominating eigenvector $\mathbf{v}_1$ across all columns. They bound this distance using the second largest eigenvalue $\lambda_2^\mathbf{\Tilde{A}}$ and the largest singular value $\sigma_1^\mathbf{W}$ of $\mathbf{W}$. Intuitively, each aggregation step $\mathbf{A}$ reduces this distance, while $\mathbf{W}$ can increase the distance arbitrarily. Thus, they consequently claim potential over-separation of node representations when $\sigma_1^\mathbf{W}>\frac{1}{\lambda_2^\mathbf{A}}$.
\citep{cai2020note} introduced the Dirichlet energy
\begin{equation}
    E(\mathbf{X}) = \mathrm{tr}(\mathbf{X}^T\mathbf{\Delta X}) = \frac{1}{2}\sum_{(i,j)\in \mathcal{E}}\Big\lVert\frac{\mathbf{x}_i}{\sqrt{d_i}}-\frac{\mathbf{x}_j}{\sqrt{d_j}}\Big\rVert^2_2
\end{equation}
as an efficient and interpretable metric to quantify over-smoothing. Their study considers the GCN with aggregation matrix $\mathbf{\Tilde{A}} = \mathbf{D}^{-1/2}\mathbf{A}\mathbf{D}^{-1/2}$ and the symmetrically normalized graph Laplacian $\mathbf{\Delta} = \mathbf{I} - \mathbf{\Tilde{A}}$, as the dominating eigenvector corresponds to its nullspace. A low energy value corresponds to similar neighboring node states.
Similarly to \citet{Oono2020Graph}, they provided the bound
\begin{equation}
    E(\mathbf{AXW}) \leq \left(\lambda_2^\mathbf{A}\right)^2\left(\sigma_1^\mathbf{W}\right)^2E(\mathbf{X})
\end{equation}
for each convolution and prove an exponential convergence to zero for the GCN. 
As their proof again only holds in case $\sigma_1^\mathbf{W}\leq \frac{1}{\lambda_2^\mathbf{A}}$, they similarly claim potential over-separation.
\citet{zhou2021dirichlet} provide a lower bound on the energy to show that the Dirichlet energy can go to infinite. 

Another branch of recent research~\citep{rusch2023gradient,rusch2022graph,rusch2023survey,wu2024demystifying} defines over-smoothing as the exponential convergence of $\mathbf{X}^{(k)}$ to a state with a constant value in each column. This contradicts the aforementioned studies as the dominant eigenvector of the aggregation matrix~\citep{Oono2020Graph,cai2020note,zhou2021dirichlet} is not always a constant vector, e.g., for the GCN~\citep{kipf2017semisupervised}. 

These inconsistencies were recently clarified as the necessity to consider the normalized state was shown \citep{digiovanni2023understanding,maskey2024fractional,roth2023rank}. When not considering the normalized state, the norm of $\mathbf{X}^{(k)}$ going to zero can be wrongly interpreted as a convergence to a constant state~\citep{digiovanni2023understanding,roth2023rank}. In the other direction, even when $\mathbf{X}^{(k)}$ is dominated by the dominant eigenvector, the Dirichlet energy of the unnormalized state wrongly indicates over-separation.

As the theory and the corresponding proofs are lengthy and complex, several recent studies still claim over-smoothing as convergence to a constant state or do not consider the normalized state. This work greatly simplifies the theory behind over-smoothing to make the theoretical backgrounds more accessible.

\section{Power Iteration}
As the proof for over-smoothing of graph convolutions will be a special case, we first provide the detailed proof for the well-known power iteration~\citep{kowalewski1909einfuhrung,miintz1913solution,mises1929praktische}. Power iteration refers to the process where a vector, when repeatedly multiplied by a matrix, gets dominated by an eigenvector of the matrix that corresponds to the eigenvalue with the largest magnitude.
The proof we provide mostly follows~\citep{knabner2017lineare}, but similar proofs are available in many textbooks.
For any square matrix $\mathbf{M}$, its eigenvalues are denoted by $\lambda_1^\mathbf{M},\dots,\lambda_n^\mathbf{M}$ and are sorted descending by their magnitude, i.e., $|\lambda_{i}^\mathbf{M}|\geq |\lambda_{i+1}^\mathbf{M}|$. 

\begin{proposition}\label{prop:power}(Power Iteration~\citep{knabner2017lineare}) Let $\mathbf{S}\in\mathbb{R}^{p\times p}$ be a matrix with $|\lambda_1^\mathbf{S}| > |\lambda_2^\mathbf{S}|$ and $\mathbf{v}_1^\mathbf{S}\in\mathbb{R}^p$ be an eigenvector corresponding to $\lambda_1^\mathbf{S}$. 
Further, let $\mathbf{x}_0\in\mathbb{R}^q$ be a vector that has a non-zero component $c_1$ in direction $\mathbf{v}_1^\mathbf{S}$. Then,
    \begin{equation}
        \frac{\mathbf{S}^k\mathbf{x}_0}{\|\mathbf{S}^k\mathbf{x}_0\|} = \beta_k\mathbf{v}_1^\mathbf{S} + \mathbf{r}_k
    \end{equation}
    for some $\mathbf{r}_k\in\mathbb{R}^p$ with $\lim_{k\to\infty}\|\mathbf{r}_k\| = 0$ and $\beta_k = \frac{c_1}{|c_1|} \left(\frac{\lambda_1^\mathbf{S}}{|\lambda_1^\mathbf{S}|}\right)^k \frac{1}{\|\mathbf{v}_1^\mathbf{S}\|}\in\mathbb{R}$.
\end{proposition}
 \begin{proof}
 Let $\mathbf{S}=\mathbf{VJV}^{-1}$ be its Jordan decomposition, where $\mathbf{J}\in\mathbb{C}^{p\times p}$ is a block diagonal matrix containing the eigenvalues on its diagonal and $\mathbf{V}\in\mathbb{C}^{p\times p}$ contains the generalized eigenvectors as columns. As the generalized eigenvectors form a basis of $\mathbf{R}^{p}$, $\mathbf{x}_0$ can be decomposed as $\mathbf{x}_0=c_1\mathbf{v}_1^\mathbf{S} + \dots + c_p\mathbf{v}_p^\mathbf{S}$ into a linear combination. This allows the following equalities:
 \begin{equation}
 \label{eq:power_proof1}
     \begin{split}
        \frac{\mathbf{S}^k\mathbf{x}_0}{\|\mathbf{S}^k\mathbf{x}_0\|} 
        &= \frac{(\mathbf{VJV}^{-1})^k(c_1\mathbf{v}_1^\mathbf{S} + \dots c_n\mathbf{v}_n^\mathbf{S})}{\|(\mathbf{VJV}^{-1})^k(c_1\mathbf{v}_1^\mathbf{S} + \dots c_n\mathbf{v}_n^\mathbf{S})\|} \\
        &= \frac{\mathbf{VJ}^k(c_1\mathbf{e}_1 + \dots c_n\mathbf{e}_n)}{\|\mathbf{VJ}^k(c_1\mathbf{e}_1 + \dots c_n\mathbf{e}_n)\|} \\
        &= \frac{c_1}{|c_1|}\left(\frac{\lambda_1^\mathbf{S}}{|\lambda_1^\mathbf{S}|}\right)^k\frac{\mathbf{V}(\frac{1}{\lambda_1^\mathbf{S}}\mathbf{J})^k\frac{1}{c_1}(c_1\mathbf{e}_1 + \dots c_n\mathbf{e}_n)}{\|\mathbf{V}(\frac{1}{\lambda_1^\mathbf{S}}\mathbf{J})^k\frac{1}{c_1}(c_1\mathbf{e}_1 + \dots c_n\mathbf{e}_n)\|} \\
    \end{split}
    \end{equation}
    The second equation uses the fact $\mathbf{V}^{-1}\mathbf{v}_k^\mathbf{S} = \mathbf{e}_k$, i.e., the natural basis vector pointing in direction $k$. As $\mathbf{J}$ is normalized by its unique largest entry $\lambda_1^\mathbf{S}$, it converges to 
    \begin{equation}
        \lim_{k\to\infty}\left(\frac{1}{\lambda_1^\mathbf{S}}\mathbf{J}\right)^k = \begin{bmatrix}
            1 &  & & \\
             & 0 & & \\
             & & \ddots & \\
             & & & 0
        \end{bmatrix}\, .
    \end{equation}
    Equation~\ref{eq:power_proof1} then simplifies to
    \begin{equation}
    \frac{c_1}{|c_1|}\left(\frac{\lambda_1}{|\lambda_1|}\right)^k\frac{\mathbf{V}(\frac{1}{\lambda_1}\mathbf{J})^k\frac{1}{c_1}(c_1\mathbf{e}_1 + \dots c_n\mathbf{e}_n)}{\|\mathbf{V}(\frac{1}{\lambda_1}\mathbf{J})^k\frac{1}{c_1}(c_1\mathbf{e}_1 + \dots c_n\mathbf{e}_n)\|} = \frac{c_1}{|c_1|}\left(\frac{\lambda_1^\mathbf{S}}{|\lambda_1^\mathbf{S}|}\right)^k \frac{\mathbf{v}_1^\mathbf{S}}{\|\mathbf{v}_1^\mathbf{S}\|} + \mathbf{r}_k
    \end{equation}
    with $\lim_{k\to\infty}\|\mathbf{r}_k\| = 0$. It converges to $\frac{\mathbf{v}_1^\mathbf{S}}{\|\mathbf{v}_1^\mathbf{S}\|}$ iff $\lambda_1^\mathbf{S} > 0$.
 \end{proof}

 Note that $|\lambda_1^\mathbf{S}| > |\lambda_2^\mathbf{S}|$ holds for almost every matrix $\mathbf{S}$ with respect to the Lebesgue measure~\citep{tao2012topics}.

\section{Graph Convolutions as Power Iteration}
We now show that this proof can be applied to all graph convolutions of the form given by Eq.~\ref{eq:update}. We express these graph convolutions in vector notation~\citep{schacke2004kronecker}
\begin{equation}
    \mathrm{vec}(\mathbf{AXW}) = (\mathbf{W}^T\otimes \mathbf{A})\mathrm{vec}(\mathbf{X}) = \mathbf{S}\mathbf{x}_0
\end{equation}
using the Kronecker product $\otimes$ that is defined as $\mathbf{A}\otimes \mathbf{B} = \begin{bmatrix}
    a_{11}\mathbf{B} & \dots & a_{1n}\mathbf{B} \\
    \vdots & \ddots & \vdots \\
    a_{m1}\mathbf{B} & \dots & a_{mn}\mathbf{B}
\end{bmatrix}$. This formulation is commonly used to study over-smoothing~\citep{digiovanni2023understanding,NEURIPS2023_2a514213,roth2023rank} and other properties of graph convolutions~\citep{gu2020implicit,roth2022transforming,di2023does}.
The Kronecker product has a key spectral property affecting power iteration: All eigenvectors $\mathbf{v}_{ij}^\mathbf{S}=\mathbf{v}_i^\mathbf{(W^T)}\otimes\mathbf{v}_j^\mathbf{A}$ of $\mathbf{W}^T\otimes\mathbf{A}$ are Kronecker products of the eigenvectors of $\mathbf{A}$ and $\mathbf{W}^T$ with corresponding eigenvalue $\lambda_i^\mathbf{A}\lambda_j^\mathbf{W}$~\citep{schacke2004kronecker}. This lets us state the reason behind over-smoothing in a clearer way than in previous works by substituting $\mathbf{v}_1^\mathbf{S}$:

\begin{proposition}\label{prop:power_kron}(Power Iteration with a Kronecker Product) Let $\mathbf{S}=\mathbf{W}\otimes\mathbf{A}\in\mathbb{R}^{(n\cdot d)\times (n\cdot d)}$ for any $\mathbf{W}\in\mathbb{R}^{d\times d}$ and $\mathbf{A}\in\mathbb{R}^{d\times d}$ with $|\lambda_1^\mathbf{S}| > |\lambda_2^\mathbf{S}|$. Let $\mathbf{v}_1^\mathbf{A}, \mathbf{v}_1^\mathbf{W}$ be two eigenvectors corresponding to $\lambda_1^\mathbf{A}$ and $\lambda_1^\mathbf{W}$, respectively.
Further, let $\mathbf{x}_0\in\mathbb{R}^{n\cdot d}$ be a vector that has a non-zero component $c_1$ in direction $\mathbf{v}_1^\mathbf{S}=\mathbf{v}_1^\mathbf{W}\otimes\mathbf{v}_1^\mathbf{A}$. Then,
    \begin{equation}
    \begin{split}
        \frac{(\mathbf{W}\otimes \mathbf{A})^k\mathbf{x}_0}{\|(\mathbf{W}\otimes \mathbf{A})^k\mathbf{x}_0\|} &= \beta_k\cdot\mathbf{v}_1^\mathbf{W}\otimes\mathbf{v}_1^\mathbf{A} + \mathbf{r}_k
    \end{split}
    \end{equation}
    for some $\mathbf{r}_k\in\mathbb{R}^{n\cdot d}$ with $\lim_{k\to\infty}\|\mathbf{r}_k\| = 0$ and $\beta_k = \frac{c_1}{|c_1|} \frac{\left(\frac{\lambda_1^{\mathbf{A}}\lambda_1^{\mathbf{W}}}{|\lambda_1^{\mathbf{A}}\lambda_1^{\mathbf{W}}|}\right)^k}{\|\mathbf{v}_1^\mathbf{W}\otimes\mathbf{v}_1^\mathbf{A}\|}\in\mathbb{R}$.
\end{proposition}

 \begin{proof}
     Given that $|\lambda_1^\mathbf{S}| > |\lambda_2^\mathbf{S}|$, and $\lambda_{i\cdot j}^\mathbf{S} = \lambda_i^\mathbf{A}\cdot\lambda_j^\mathbf{W}$ for all $0<i<n$ and $0<j<d$, we have $|\lambda_1^\mathbf{A}|>|\lambda_2^\mathbf{A}|$ and $|\lambda_1^\mathbf{W}|>|\lambda_2^\mathbf{W}|$. The corresponding eigenvector $\mathbf{v}_1^\mathbf{S} = \mathbf{v}_1^\mathbf{A}\otimes\mathbf{v}_1^\mathbf{W}$ is the Kronecker product of the corresponding eigenvectors of $\mathbf{A}$ and $\mathbf{W}$. Substituting these in Proposition~\ref{prop:power} results in our Proposition~\ref{prop:power_kron}.
 \end{proof}

Extending Proposition~\ref{prop:power_kron} for any $\mathbf{W}$ and possibly repeated $\lambda_1^\mathbf{W}$ is similar, as all generalized eigenvectors of $\mathbf{W}\otimes\mathbf{A}$ corresponding to $\lambda_1^\mathbf{S}$ are of the form $\mathbf{u}\otimes\mathbf{v}_1^\mathbf{A}$ for different $\mathbf{u}$. To simplify this work, we provide this proof as Proposition~\ref{prop:power_kron_app} in Appendix~\ref{sec:appendix}.
The implications of Proposition~\ref{prop:power_kron} become clearer when looking into its matrix form: 

\begin{remark}(Power Iteration with a Kronecker Product in Matrix Notation)
Stating Proposition~\ref{prop:power_kron} in matrix notation leads to
    \begin{equation}
        \frac{\mathbf{A}^k\mathbf{XW}^k}{\|\mathbf{A}^k\mathbf{XW}^k\|} = \beta_k\mathbf{v}_1^\mathbf{A}\left(\mathbf{v}_1^\mathbf{W}\right)^T + \mathbf{R}_k
    \end{equation}
    for $\mathrm{vec}(\mathbf{X}) = \mathbf{x}_0$ and some $\mathbf{R}_k$ with $\lim_{k\to\infty} \|\mathbf{R}_k\| = 0$ and $\beta_k = \frac{c_1}{|c_1|} \frac{\left(\frac{\lambda_1^{\mathbf{A}}\lambda_1^{\mathbf{W}}}{|\lambda_1^{\mathbf{A}}\lambda_1^{\mathbf{W}}|}\right)^k}{\|\mathbf{v}_1^\mathbf{W}\otimes\mathbf{v}_1^\mathbf{A}\|}\in\mathbb{R}$.
\end{remark}

Any graph convolution of this form amplifies the same signal across all feature columns, and the state gets closer to a rank one matrix, with each column becoming a multiple of $\mathbf{v}_1^\mathbf{A}$.
This phenomenon was termed rank collapse, but a definition is still missing~\citep{roth2023rank}. A comprehensive definition must consider the normalized representations and be independent of one specific vector. We introduce the following definition:

\begin{definition}(Rank Collapse)
    A sequence of matrices $\mathbf{X}^{(1)},\dots,\mathbf{X}^{(k)}\in\mathbb{R}^{n\times d}$ is said to suffer from \textbf{rank collapse} if there exists a sequence of rank-one matrices $\mathbf{Y}^{(1)},\dots,\mathbf{Y}^{(k)}\in\mathbb{R}^{n\times d}$ such that
    \begin{equation}
        \lim_{k\to\infty} \left\|\frac{\mathbf{X}^{(k)}}{\|\mathbf{X}^{(k)}\|} - \mathbf{Y}^{(k)}\right\| = 0
    \end{equation}
\end{definition}

It is commonly referred to as over-smoothing as they only consider stochastic aggregation functions or symmetrically normalized adjacency matrices.
Their eigenvector $\mathbf{v}_1^\mathbf{A}$ is a smooth vector for typical choices of $\mathbf{A}$, e.g., it is the vector of all ones $\mathbf{v}_1^\mathbf{A} = \mathbb{1}$ for the (weighted) mean aggregation, and $\mathbf{v}_1^\mathbf{A} = \mathbf{D}^{\frac{1}{2}}\mathbb{1}$ for the symmetrically normalized adjacency matrix~\citep{von2007tutorial}. We thus define over-smoothing as a special case of rank collapse:

\begin{definition}(Over-Smoothing)
    A sequence of matrices $\mathbf{X}^{(1)},\dots,\mathbf{X}^{(k)}\in\mathbb{R}^{n\times d}$ is said to suffer from \textbf{over-smoothing} if it suffers from rank collapse and the rank-one matrices are of the form $\mathbf{Y}^{(l)} = \mathbf{1}\mathbf{c}_{(l)}^T$ or $\mathbf{Y} = \mathbf{D}^{\frac{1}{2}}\mathbf{1}\mathbf{c}_{(l)}^T$ for any $\mathbf{c}_{(l)}\in\mathbb{R}^{d}$.
\end{definition}
To quantify the degree of over-smoothing, a frequently used metric is the Dirichlet energy~\citep{cai2020note} 
\begin{equation}
    E\left(\frac{\mathbf{X}}{\|\mathbf{X}\|}\right) = \mathrm{tr}\left(\frac{\mathbf{X}}{\|\mathbf{X}\|}\mathbf{\Delta}\frac{\mathbf{X}}{\|\mathbf{X}\|}\right)\, ,
\end{equation}
with $\mathbf{\Delta} = \mathbf{D} - \mathbf{A}$ or $\mathbf{\Delta} = \mathbf{I}_n - \mathbf{D}^{-1/2}\mathbf{A}\mathbf{D}^{-1/2}$, depending on case of over-smoothing. The Dirichlet energy converges to zero for methods utilizing the corresponding aggregation function as $\mathbf{v}_1$ is in the nullspace of $\mathbf{\Delta}$, i.e., $\mathbf{\Delta}\mathbf{v}_1 = \mathbf{0}$. However, this requires different $\mathbf{\Delta}$ for different aggregation functions, as $\mathbf{v}_1^\mathbf{A}$ may be different.

In order to quantify the more general phenomenon of rank collapse, we need to find a suitable sequence of rank-one matrices. The singular vectors corresponding to the largest singular value give the closest rank-one approximation of a given matrix. However, finding singular vectors is computationally expensive. As we are mainly interested in whether $\mathbf{X}$ converges to zero and not the exact value, we will instead utilize a row and column vector of the given $\mathbf{X}$. Any row and column will be sufficient if $\mathbf{X}$ is close to a rank one matrix. We consider the row and column vectors with the largest norm for numerical stability. This leads to our newly proposed rank-one distance metric:

\begin{definition}(Rank-one Distance (ROD))
Let $\mathbf{X}\in\mathbb{R}^{p\times q}$ be any matrix. We define the row with the largest norm as $\mathbf{v} = \max_i \|\mathbf{x}_i\|\in\mathbb{R}^p$ and the column index corresponding the largest norm as $j = \mathrm{arg}\max_i \|\mathbf{x}_{:,i}\|\in\mathbb{R}^q$. To account for the correct signs, we utilize the $j$-th column vector $\mathbf{u} = \mathbf{x}_{:,j}$ if $v_j > 0$ or its negated version $\mathbf{u} = -\mathbf{x}_{:,j}$ otherwise. The rank-one distance of $\mathbf{X}$ is then defined as
 \begin{equation}
\mathrm{ROD}(\mathbf{X}) = \left\|\frac{\mathbf{X}}{\|\mathbf{X}\|} - \frac{\mathbf{u}\mathbf{v}^T}{\|\mathbf{u}\mathbf{v}^T\|}\right\|     \, .
 \end{equation}
\end{definition}

As a generalization of over-smoothing, a Dirichlet energy of zero implies a ROD of zero, i.e., $E(\mathbf{X}/\|\mathbf{X}\|) = 0 \Rightarrow \mathrm{ROD}(\mathbf{X}) = 0$. Similarly to \citet{roth2023rank}, our theory also explains how to prevent over-smoothing and rank collapse. It needs to be ensured that the graph convolution is not a Kronecker product, i.e., a single aggregation and transformation matrix. One direction is to operate on multiple computational graphs $\mathbf{A}_1,\dots,\mathbf{A}_l$ with distinct feature transformations $\mathbf{W}_1,\dots,\mathbf{W}_l$:
\begin{equation}\begin{split}
\mathbf{S}\mathrm{vec}(\mathbf{X}) 
&= (\mathbf{W}_1\otimes\mathbf{A}_1 + \dots + \mathbf{W}_l\otimes\mathbf{A}_l)\mathrm{vec}(\mathbf{X}) \\
&= \mathrm{vec}(\mathbf{A}_1\mathbf{X}\mathbf{W}_1^T + \dots + \mathbf{A}_l\mathbf{X}\mathbf{W}_l^T)\, .
\end{split}
\end{equation}
As the eigenvectors of sums of Kronecker products can be linearly independent across the corresponding columns, different signals can get amplified for each feature column.

\section{Experimental Validation}
\begin{figure*}[tb]
         \centering
        \def\svgwidth{0.90\textwidth}
     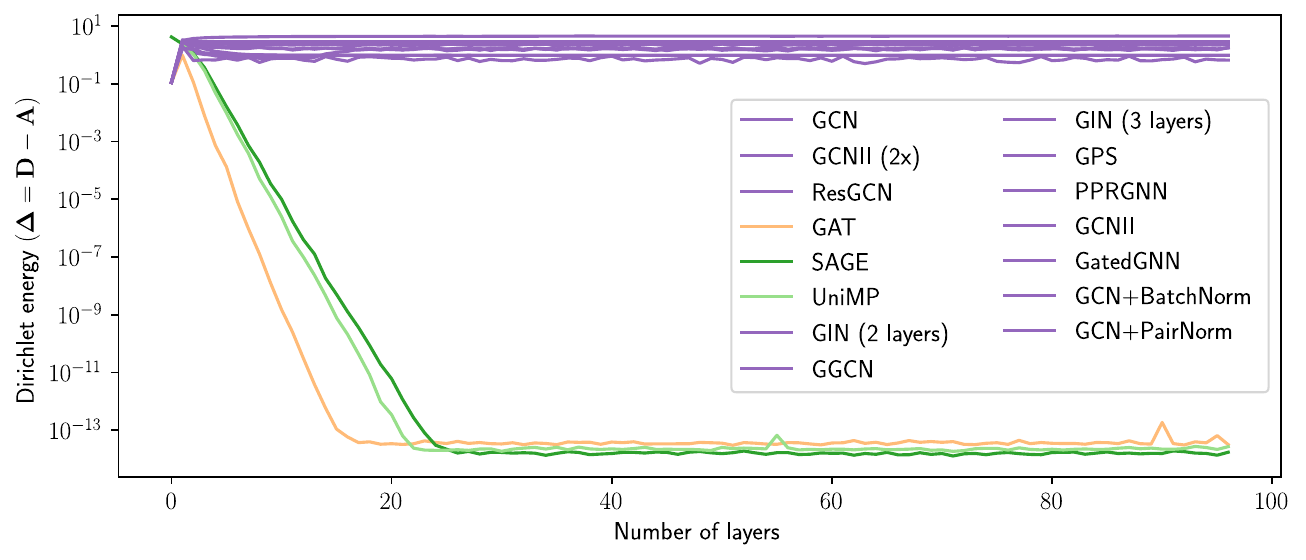
     \caption{Change in Dirichlet energy using the unnormalized graph Laplacian when increasing the number of layers for randomly initialized models. Mean values over $50$ random seeds.}
      \label{fig:dirichlet_unnorm}     
\end{figure*}
Given our novel definition and corresponding metric, we evaluate various well-established graph convolutions in terms of their ability to avoid rank collapse. Our implementation based on PyTorchGeometric~\citep{fey2019fast} is available online~\footnote{\href{https://github.com/roth-andreas/simplifying-over-smoothing}{https://github.com/roth-andreas/simplifying-over-smoothing}}.

\paragraph{Methods}
As base message-passing methods, we evaluate the dynamics of the graph convolutional network~\citep{kipf2017semisupervised} and the graph attention network (GAT)~\citep{veličković2018graph}. These are known to suffer from over-smoothing, so they also suffer rank collapse. 
The other methods we consider are not generally known to cause over-smoothing or are specifically designed to prevent it.  
While previous theory shows that negative edge weights and similarly residual connections can avoid over-smoothing~\citep{bo2021beyond,yan2022two,digiovanni2023understanding}, the connection to power iteration implies that these cannot prevent rank collapse. We evaluate the combination of the GCN and a residual connection (ResGCN), which adds the previous state to the output of each convolution. Similarly, we evaluate the SAGE convolution~\citep{hamilton2017inductive}, which additionally applies a linear transformation to the previous state. We also consider a commonly employed method that allows negative attention weights, namely the generalized GCN (GGCN)~\citep{yan2022two}.
\begin{figure*}[tb]
         \centering
        \def\svgwidth{0.95\textwidth}
     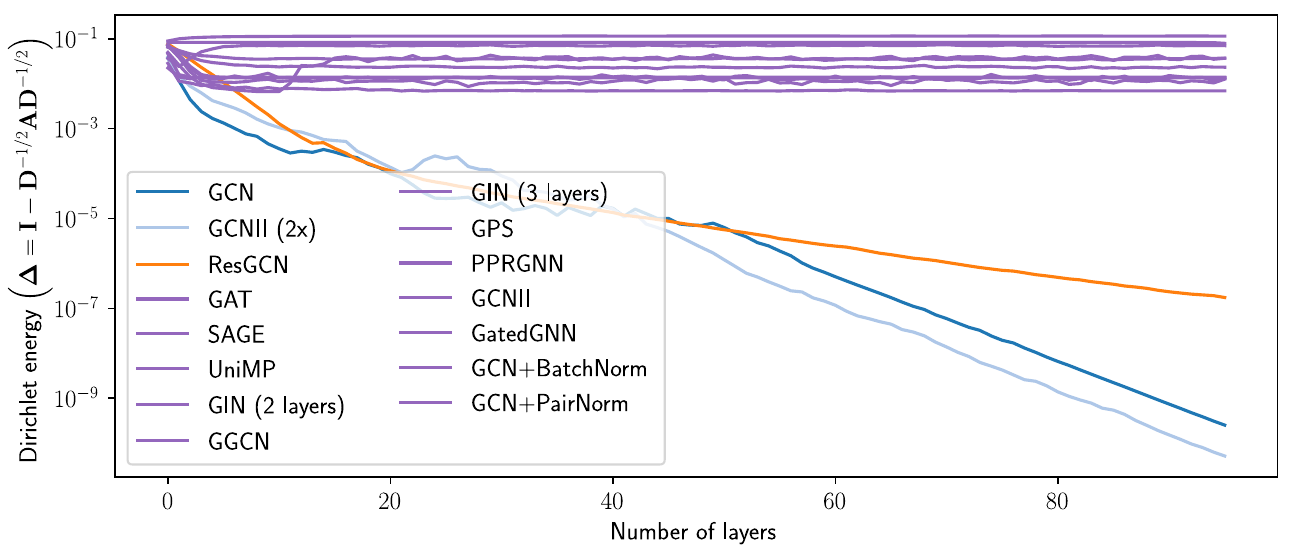
     \caption{Change in Dirichlet energy using the symmetrically normalized graph Laplacian when increasing the number of layers for randomly initialized models. Mean values over $50$ random seeds.}
      \label{fig:dirichlet_symm}     
\end{figure*}
\begin{figure*}[tb]
         \centering
        \def\svgwidth{0.95\textwidth}
     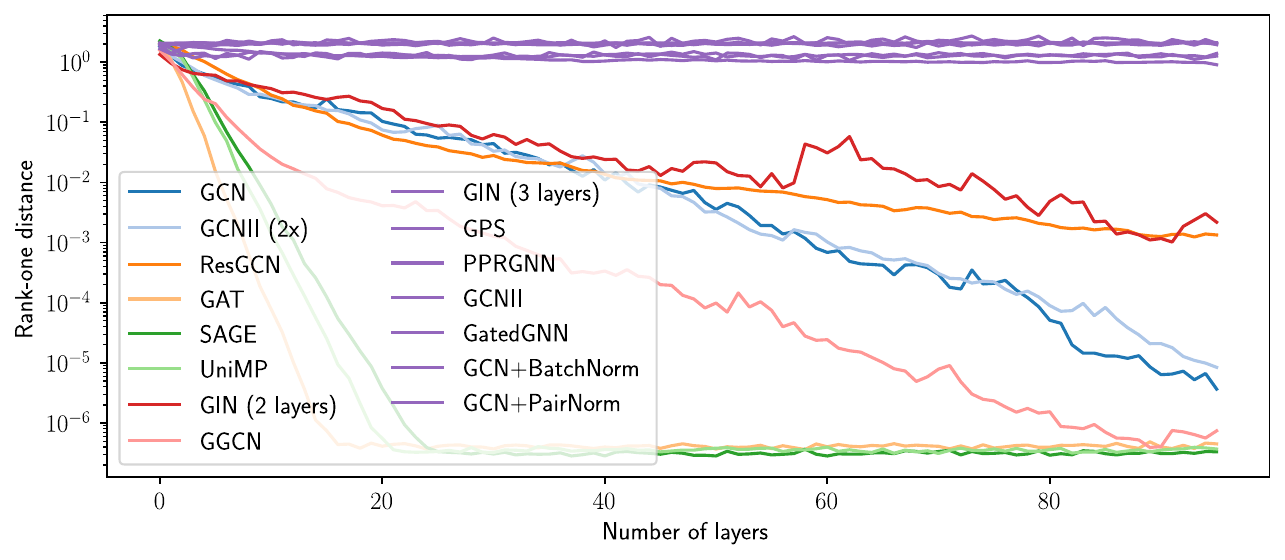
     \caption{Change of the rank-one distance when increasing the number of layers for randomly initialized models. Mean values over 50 random seeds. }
      \label{fig:rank_one}     
\end{figure*}
Another direction that aims to prevent over-smoothing is based on combining the output of each iteration with the initial state, referred to as an initial residual connection~\citep{chen2020simple} or a restart term~\citep{roth2022transforming}. We evaluate the GCNII~\citep{chen2020simple}. While previous work argued that this prevents over-smoothing for the unnormalized state. Given the critical importance of considering the normalized state, we want to evaluate whether this property still holds. As the magnitude of the representations may grow in each iteration, the constant influence of the initial state may become negligible.
Towards this end, we consider two versions: one with the regular parameter initialization (GCNII) and one for which all parameters are scaled by a factor of two (GCNII 2x).
As an alternative approach, we also consider the personalized page rank graph neural network (PPRGNN)~\citep{roth2022transforming}, a formulation that provably converges to a steady state with increased depth. 
Gating mechanisms allow a node to retain its representation by limiting the mixing with neighboring states using learnable gating functions. Here, we evaluate the gated graph neural network (GatedGNN)~\citep{yujia2016}. 
Normalization layers were also shown to prevent over-smoothing. Here, we evaluate PairNorm~\citep{Zhao2020PairNorm} (GCN+PairNorm) and BatchNorm~\citep{ioffe2015batch} (GCN+BatchNorm), combined with the GCN by their ability to mitigate rank collapse.
The graph isomorphism network~\citep{xu2018how} applies a multi-layer non-linear feature transformation to each state. While this was designed to allow for injective mappings of multisets to achieve maximal expressivity, this should allow resulting representations to be linearly independent. We evaluate a version using a multi-layer perceptron (MLP) with two layers (GIN 2 layers) and a version using an MLP with three layers (GIN 3 layers). As global methods, we consider the unified message passaging model (UniMP)~\citep{shi2021} and the general, powerful, scalable (GPS) graph transformer~\citep{rampavsek2022recipe} as two graph transformer variations.

\paragraph{Setting}
We track the rank-one distance for all methods on the KarateClub dataset~\citep{doi:10.1086/jar.33.4.3629752}, which is a small and undirected graph consisting of $34$ nodes and $156$ edges. We initialize each node with randomly assigned features following a normal distribution. For each method and for each iteration, the corresponding graph convolution followed by a ReLU-activation is applied. The rank-one distance is computed after each iteration. For direct comparison, we also compute the Dirichlet energy based on the unnormalized graph Laplacian $\bm{\Delta} = \mathbf{D} - \mathbf{A}$ and the symmetrically normalized graph Laplacian $\bm{\Delta} = \mathbf{I} - \mathbf{D}^{-1/2}\mathbf{A}\mathbf{D}^{-1/2}$. We repeat this process for $96$ iterations as several methods become too unstable at further depth, i.e., feature magnitudes explode or diminish. Runs for each method are repeated for $50$ random seeds. Mean values are reported.

\paragraph{Results}

The changes in the Dirichlet energy using the unnormalized graph Laplacian are shown in Figure~\ref{fig:dirichlet_unnorm}. This Dirichlet energy converges to zero only for GAT, SAGE, and UniMP, which all use weighted mean aggregation. In Figure~\ref{fig:dirichlet_symm}, we visualize the changes in Dirichlet energy using the symmetrically normalized graph Laplacian. For GCN, GCNII (2x), and ResGCN, this metric converges to zero as these methods utilize the symmetrically normalized adjacency matrix for aggregation.
In Figure~\ref{fig:rank_one}, we present the rank-one distance (ROD) for all methods. The results confirm that ROD converges to zero for all methods for which the Dirichlet energy with either graph Laplacian converges to zero. In addition, we find GIN (2 layers) and GGCN to suffer from rank collapse. While these methods use different aggregation functions, ROD captures the underlying issue of rank collapse. Different rates of convergence are notable, particularly between methods using the mean aggregation (GAT, SAGE, UniMP) and methods using other aggregation functions (GCN, GGCN, GCNII, ResGCN, and GIN). However, all these methods cause the rank to converge to one across all $50$ random initializations. Methods for which we do not observe rank collapse use normalization techniques (GCN+PairNorm and GCN+BatchNorm), prevent unlimited depth (PPRGNN, GCNII, and GatedGNN), and use non-linear feature transformations (GPS and GIN (3-layer)).

\section{Conclusion}
We have shown that over-smoothing is a special case of power iteration, with the dominant eigenvector of graph convolutions $\mathbf{W}\otimes\mathbf{A}$ taking the form $\mathbf{v}_1^\mathbf{W}\otimes\mathbf{v}_1^\mathbf{A}$. As given in power iteration, normalization is required, and the limit distribution is not always the constant vector, as it depends on the dominant eigenvector of $\mathbf{A}$. Based on our novel definition of rank collapse and the corresponding rank-one distance, we identified several methods suffering from rank collapse that were previously considered to prevent over-smoothing. We empirically found three general directions to prevent rank collapse: Normalization layers, limiting the effective depth of a model using a restart term or a gating mechanism, and complex non-linear feature transformations. These are promising avenues for further consideration. In addition, to solve rank collapse in the message-passing steps itself, the theory indicates that it needs to be ensured that the dominant eigenvector $\mathbf{v}_1^\mathbf{S}$ is not a simple Kronecker product so that it can amplify different signals across feature columns. 

\begin{ack}
  This research has been funded by the Federal Ministry of Education and Research of Germany under grant no. 01IS22094E WEST-AI. 
\end{ack}

\medskip

\bibliography{references}

\begin{thebibliography}{41}
\providecommand{\natexlab}[1]{#1}
\providecommand{\url}[1]{\texttt{#1}}
\expandafter\ifx\csname urlstyle\endcsname\relax
  \providecommand{\doi}[1]{doi: #1}\else
  \providecommand{\doi}{doi: \begingroup \urlstyle{rm}\Url}\fi

\bibitem[Bo et~al.(2021)Bo, Wang, Shi, and Shen]{bo2021beyond}
D.~Bo, X.~Wang, C.~Shi, and H.~Shen.
\newblock Beyond low-frequency information in graph convolutional networks.
\newblock In \emph{Proceedings of the AAAI conference on artificial intelligence}, volume~35, pages 3950--3957, 2021.

\bibitem[Cai and Wang(2020)]{cai2020note}
C.~Cai and Y.~Wang.
\newblock A note on over-smoothing for graph neural networks.
\newblock \emph{arXiv preprint arXiv:2006.13318}, 2020.

\bibitem[Chen et~al.(2020{\natexlab{a}})Chen, Lin, Li, Li, Zhou, and Sun]{chen2020measuring}
D.~Chen, Y.~Lin, W.~Li, P.~Li, J.~Zhou, and X.~Sun.
\newblock Measuring and relieving the over-smoothing problem for graph neural networks from the topological view.
\newblock In \emph{Proceedings of the AAAI conference on artificial intelligence}, volume~34, pages 3438--3445, 2020{\natexlab{a}}.

\bibitem[Chen et~al.(2020{\natexlab{b}})Chen, Wei, Huang, Ding, and Li]{chen2020simple}
M.~Chen, Z.~Wei, Z.~Huang, B.~Ding, and Y.~Li.
\newblock Simple and deep graph convolutional networks.
\newblock In \emph{International conference on machine learning}, pages 1725--1735. PMLR, 2020{\natexlab{b}}.

\bibitem[Di~Giovanni et~al.(2023)Di~Giovanni, Rusch, Bronstein, Deac, Lackenby, Mishra, and Veli{\v{c}}kovi{\'c}]{di2023does}
F.~Di~Giovanni, T.~K. Rusch, M.~M. Bronstein, A.~Deac, M.~Lackenby, S.~Mishra, and P.~Veli{\v{c}}kovi{\'c}.
\newblock How does over-squashing affect the power of gnns?
\newblock \emph{arXiv preprint arXiv:2306.03589}, 2023.

\bibitem[Fey and Lenssen(2019)]{fey2019fast}
M.~Fey and J.~E. Lenssen.
\newblock Fast graph representation learning with pytorch geometric.
\newblock \emph{arXiv preprint arXiv:1903.02428}, 2019.

\bibitem[Giovanni et~al.(2023)Giovanni, Rowbottom, Chamberlain, Markovich, and Bronstein]{digiovanni2023understanding}
F.~D. Giovanni, J.~Rowbottom, B.~P. Chamberlain, T.~Markovich, and M.~M. Bronstein.
\newblock Understanding convolution on graphs via energies.
\newblock \emph{Transactions on Machine Learning Research}, 2023.

\bibitem[Gu et~al.(2020)Gu, Chang, Zhu, Sojoudi, and El~Ghaoui]{gu2020implicit}
F.~Gu, H.~Chang, W.~Zhu, S.~Sojoudi, and L.~El~Ghaoui.
\newblock Implicit graph neural networks.
\newblock \emph{Advances in Neural Information Processing Systems}, 33:\penalty0 11984--11995, 2020.

\bibitem[Hamilton et~al.(2017)Hamilton, Ying, and Leskovec]{hamilton2017inductive}
W.~Hamilton, Z.~Ying, and J.~Leskovec.
\newblock Inductive representation learning on large graphs.
\newblock \emph{Advances in neural information processing systems}, 30, 2017.

\bibitem[Ioffe and Szegedy(2015)]{ioffe2015batch}
S.~Ioffe and C.~Szegedy.
\newblock Batch normalization: Accelerating deep network training by reducing internal covariate shift.
\newblock In \emph{International conference on machine learning}, pages 448--456. pmlr, 2015.

\bibitem[Kipf and Welling(2017)]{kipf2017semisupervised}
T.~N. Kipf and M.~Welling.
\newblock Semi-supervised classification with graph convolutional networks.
\newblock In \emph{International Conference on Learning Representations}, 2017.

\bibitem[Knabner and Barth(2017)]{knabner2017lineare}
P.~Knabner and W.~Barth.
\newblock \emph{Lineare Algebra}.
\newblock Springer, 2017.

\bibitem[Kowalewski(1909)]{kowalewski1909einfuhrung}
G.~Kowalewski.
\newblock \emph{Einf{\"u}hrung in die Determinantentheorie einschlie{\ss}lich der unendlichen und der Fredholmschen Determinanten}.
\newblock Veit \& comp., 1909.

\bibitem[Li et~al.(2018)Li, Han, and Wu]{li2018deeper}
Q.~Li, Z.~Han, and X.-M. Wu.
\newblock Deeper insights into graph convolutional networks for semi-supervised learning.
\newblock In \emph{Proceedings of the AAAI conference on artificial intelligence}, volume~32, 2018.

\bibitem[Li et~al.(2016)Li, Tarlow, Brockschmidt, and Zemel]{yujia2016}
Y.~Li, D.~Tarlow, M.~Brockschmidt, and R.~S. Zemel.
\newblock Gated graph sequence neural networks.
\newblock In Y.~Bengio and Y.~LeCun, editors, \emph{4th International Conference on Learning Representations, {ICLR} 2016, San Juan, Puerto Rico, May 2-4, 2016, Conference Track Proceedings}, 2016.

\bibitem[Maskey et~al.(2023{\natexlab{a}})Maskey, Paolino, Bacho, and Kutyniok]{NEURIPS2023_2a514213}
S.~Maskey, R.~Paolino, A.~Bacho, and G.~Kutyniok.
\newblock A fractional graph laplacian approach to oversmoothing.
\newblock In A.~Oh, T.~Naumann, A.~Globerson, K.~Saenko, M.~Hardt, and S.~Levine, editors, \emph{Advances in Neural Information Processing Systems}, volume~36, pages 13022--13063. Curran Associates, Inc., 2023{\natexlab{a}}.

\bibitem[Maskey et~al.(2023{\natexlab{b}})Maskey, Paolino, Bacho, and Kutyniok]{maskey2024fractional}
S.~Maskey, R.~Paolino, A.~Bacho, and G.~Kutyniok.
\newblock A fractional graph laplacian approach to oversmoothing.
\newblock \emph{Advances in Neural Information Processing Systems}, 36, 2023{\natexlab{b}}.

\bibitem[Miintz(1913)]{miintz1913solution}
L.~Miintz.
\newblock Solution direct de 1'equation seculaire et de quelques problemes analogues transcendents.
\newblock \emph{Comptes Rendus de I'Academie des Sciences, Paris}, 156:\penalty0 43--6, 1913.

\bibitem[Mises and Pollaczek-Geiringer(1929)]{mises1929praktische}
R.~Mises and H.~Pollaczek-Geiringer.
\newblock Praktische verfahren der gleichungsaufl{\"o}sung.
\newblock \emph{ZAMM-Journal of Applied Mathematics and Mechanics/Zeitschrift f{\"u}r Angewandte Mathematik und Mechanik}, 9\penalty0 (1):\penalty0 58--77, 1929.

\bibitem[Oono and Suzuki(2020)]{Oono2020Graph}
K.~Oono and T.~Suzuki.
\newblock Graph neural networks exponentially lose expressive power for node classification.
\newblock In \emph{International Conference on Learning Representations}, 2020.

\bibitem[Ramp{\'a}{\v{s}}ek et~al.(2022)Ramp{\'a}{\v{s}}ek, Galkin, Dwivedi, Luu, Wolf, and Beaini]{rampavsek2022recipe}
L.~Ramp{\'a}{\v{s}}ek, M.~Galkin, V.~P. Dwivedi, A.~T. Luu, G.~Wolf, and D.~Beaini.
\newblock Recipe for a general, powerful, scalable graph transformer.
\newblock \emph{Advances in Neural Information Processing Systems}, 35:\penalty0 14501--14515, 2022.

\bibitem[Rong et~al.(2020)Rong, Huang, Xu, and Huang]{Rong2020DropEdge}
Y.~Rong, W.~Huang, T.~Xu, and J.~Huang.
\newblock Dropedge: Towards deep graph convolutional networks on node classification.
\newblock In \emph{International Conference on Learning Representations}, 2020.

\bibitem[Roth and Liebig(2022)]{roth2022transforming}
A.~Roth and T.~Liebig.
\newblock Transforming pagerank into an infinite-depth graph neural network.
\newblock In \emph{Joint European conference on machine learning and knowledge discovery in databases}, pages 469--484. Springer, 2022.

\bibitem[Roth and Liebig(2023)]{roth2023rank}
A.~Roth and T.~Liebig.
\newblock Rank collapse causes over-smoothing and over-correlation in graph neural networks.
\newblock In \emph{The Second Learning on Graphs Conference}, 2023.

\bibitem[Rusch et~al.(2022)Rusch, Chamberlain, Rowbottom, Mishra, and Bronstein]{rusch2022graph}
T.~K. Rusch, B.~Chamberlain, J.~Rowbottom, S.~Mishra, and M.~Bronstein.
\newblock Graph-coupled oscillator networks.
\newblock In \emph{International Conference on Machine Learning}, pages 18888--18909. PMLR, 2022.

\bibitem[Rusch et~al.(2023{\natexlab{a}})Rusch, Bronstein, and Mishra]{rusch2023survey}
T.~K. Rusch, M.~M. Bronstein, and S.~Mishra.
\newblock A survey on oversmoothing in graph neural networks.
\newblock \emph{arXiv preprint arXiv:2303.10993}, 2023{\natexlab{a}}.

\bibitem[Rusch et~al.(2023{\natexlab{b}})Rusch, Chamberlain, Mahoney, Bronstein, and Mishra]{rusch2023gradient}
T.~K. Rusch, B.~P. Chamberlain, M.~W. Mahoney, M.~M. Bronstein, and S.~Mishra.
\newblock Gradient gating for deep multi-rate learning on graphs.
\newblock In \emph{The Eleventh International Conference on Learning Representations}, 2023{\natexlab{b}}.

\bibitem[Schacke(2004)]{schacke2004kronecker}
K.~Schacke.
\newblock On the kronecker product.
\newblock \emph{Master's thesis, University of Waterloo}, 2004.

\bibitem[Scholkemper et~al.(2024)Scholkemper, Wu, Jadbabaie, and Schaub]{scholkemper2024residual}
M.~Scholkemper, X.~Wu, A.~Jadbabaie, and M.~Schaub.
\newblock Residual connections and normalization can provably prevent oversmoothing in gnns.
\newblock \emph{arXiv preprint arXiv:2406.02997}, 2024.

\bibitem[Shi et~al.(2021)Shi, Zhengjie, Feng, Zhong, Wang, and Sun]{shi2021}
Y.~Shi, H.~Zhengjie, S.~Feng, H.~Zhong, W.~Wang, and Y.~Sun.
\newblock Masked label prediction: Unified message passing model for semi-supervised classification.
\newblock pages 1548--1554, 08 2021.
\newblock \doi{10.24963/ijcai.2021/214}.

\bibitem[Tao(2012)]{tao2012topics}
T.~Tao.
\newblock \emph{Topics in random matrix theory}, volume 132.
\newblock American Mathematical Soc., 2012.

\bibitem[Veličković et~al.(2018)Veličković, Cucurull, Casanova, Romero, Liò, and Bengio]{veličković2018graph}
P.~Veličković, G.~Cucurull, A.~Casanova, A.~Romero, P.~Liò, and Y.~Bengio.
\newblock Graph attention networks.
\newblock In \emph{International Conference on Learning Representations}, 2018.
\newblock URL \url{https://openreview.net/forum?id=rJXMpikCZ}.

\bibitem[Von~Luxburg(2007)]{von2007tutorial}
U.~Von~Luxburg.
\newblock A tutorial on spectral clustering.
\newblock \emph{Statistics and computing}, 17:\penalty0 395--416, 2007.

\bibitem[Wu et~al.(2023{\natexlab{a}})Wu, Ajorlou, Wu, and Jadbabaie]{NEURIPS2023_6e4cdfdd}
X.~Wu, A.~Ajorlou, Z.~Wu, and A.~Jadbabaie.
\newblock Demystifying oversmoothing in attention-based graph neural networks.
\newblock In A.~Oh, T.~Naumann, A.~Globerson, K.~Saenko, M.~Hardt, and S.~Levine, editors, \emph{Advances in Neural Information Processing Systems}, volume~36, pages 35084--35106. Curran Associates, Inc., 2023{\natexlab{a}}.

\bibitem[Wu et~al.(2023{\natexlab{b}})Wu, Ajorlou, Wu, and Jadbabaie]{wu2024demystifying}
X.~Wu, A.~Ajorlou, Z.~Wu, and A.~Jadbabaie.
\newblock Demystifying oversmoothing in attention-based graph neural networks.
\newblock \emph{Advances in Neural Information Processing Systems}, 36, 2023{\natexlab{b}}.

\bibitem[Xu et~al.(2019)Xu, Hu, Leskovec, and Jegelka]{xu2018how}
K.~Xu, W.~Hu, J.~Leskovec, and S.~Jegelka.
\newblock How powerful are graph neural networks?
\newblock In \emph{International Conference on Learning Representations}, 2019.
\newblock URL \url{https://openreview.net/forum?id=ryGs6iA5Km}.

\bibitem[Yan et~al.(2022)Yan, Hashemi, Swersky, Yang, and Koutra]{yan2022two}
Y.~Yan, M.~Hashemi, K.~Swersky, Y.~Yang, and D.~Koutra.
\newblock Two sides of the same coin: Heterophily and oversmoothing in graph convolutional neural networks.
\newblock In \emph{2022 IEEE International Conference on Data Mining (ICDM)}, pages 1287--1292. IEEE, 2022.

\bibitem[Zachary(1977)]{doi:10.1086/jar.33.4.3629752}
W.~W. Zachary.
\newblock An information flow model for conflict and fission in small groups.
\newblock \emph{Journal of Anthropological Research}, 33\penalty0 (4):\penalty0 452--473, 1977.
\newblock \doi{10.1086/jar.33.4.3629752}.

\bibitem[Zhao and Akoglu(2020)]{Zhao2020PairNorm}
L.~Zhao and L.~Akoglu.
\newblock Pairnorm: Tackling oversmoothing in gnns.
\newblock In \emph{International Conference on Learning Representations}, 2020.

\bibitem[Zhou et~al.(2020)Zhou, Huang, Li, Zha, Chen, and Hu]{zhou2020towards}
K.~Zhou, X.~Huang, Y.~Li, D.~Zha, R.~Chen, and X.~Hu.
\newblock Towards deeper graph neural networks with differentiable group normalization.
\newblock \emph{Advances in neural information processing systems}, 33:\penalty0 4917--4928, 2020.

\bibitem[Zhou et~al.(2021)Zhou, Huang, Zha, Chen, Li, Choi, and Hu]{zhou2021dirichlet}
K.~Zhou, X.~Huang, D.~Zha, R.~Chen, L.~Li, S.-H. Choi, and X.~Hu.
\newblock Dirichlet energy constrained learning for deep graph neural networks.
\newblock \emph{Advances in Neural Information Processing Systems}, 34:\penalty0 21834--21846, 2021.

\end{thebibliography}
\bibliographystyle{abbrvnat}


\newpage
\appendix
\section{Appendix}
\label{sec:appendix}
\begin{proposition}\label{prop:power_kron_app}(Power Iteration with a Kronecker Product) Let $\mathbf{S}=\mathbf{W}\otimes\mathbf{A}$ for $\mathbf{W}\in\mathbb{R}^{d\times d}$ and $\mathbf{A}\in\mathbb{R}^{n\times n}$ with $|\lambda_1^\mathbf{A}| > |\lambda_2^\mathbf{A}|$. Let $\mathbf{v}_1^\mathbf{A}$ be an eigenvector corresponding to $\lambda_1^\mathbf{A}$.
Further, let $\mathbf{x}_0\in\mathbb{R}^{n\cdot d}$ be any vector that has a non-zero component in the direction of a generalized eigenvector $\mathbf{v}_1^\mathbf{S}$ corresponding to $\lambda_1^\mathbf{S}$. Then,
    \begin{equation}
    \begin{split}
        \frac{(\mathbf{W}\otimes \mathbf{A})^k\mathbf{x}_0}{\|(\mathbf{W}\otimes \mathbf{A})^k\mathbf{x}_0\|} &= \beta_k\cdot\mathbf{u}\otimes\mathbf{v}_1^\mathbf{A} + \mathbf{r}_k
    \end{split}
    \end{equation}
    for some $\mathbf{r}_k\in\mathbb{R}^{n\cdot d}$ with $\lim_{k\to\infty}\|\mathbf{r}_k\| = 0$, bounded $\beta_k$, and some $\mathbf{u}\in\mathbb{R}^d$.
\end{proposition}

\begin{proof}
This proof is similar to the proof of Proposition~\ref{prop:power}. However, $\pm\lambda_1^\mathbf{S}$ may occur multiple times, so there can be multiple Jordan blocks corresponding to $\pm\lambda_1^\mathbf{S}$, and they can have a size larger than one. Let $p$ be the size of the largest Jordan block corresponding to $\lambda_1^\mathbf{S}$. Then, $\mathbf{J}^k$ will be dominated by $q_k = \binom{k}{p-1}\lambda_1^{S^{k-(p-1)}}$:
\begin{equation}
        \lim_{k\to\infty}\left(\frac{1}{q_k}\mathbf{J}\right)^k = \begin{bmatrix}
            0 & \dots & 0 & 1 & & & \\
             & \ddots & & 0 & & & \\
            \vdots & & & \vdots & & \\
            0 & \dots & & 0 & & \\ 
             & & & & \ddots & \\
             & & & & & 0 & \dots & 0 & 1 \\
             & & & & & & \ddots & & 0 \\
             & & & & & \vdots & & & \vdots \\
             & & & & & 0 & \dots & & 0 \\
             & & & & & & & & & 0 \\
             & & & & & & & & & & \ddots \\
             & & & & & & & & & & & 0
        \end{bmatrix}\, .
    \end{equation}
    The number of blocks containing a $1$ is determined by the number of Jordan blocks corresponding to $\pm\lambda_1^{\mathbf{S}}$ with size $p$. Let there be $i$ such blocks. We further know that all corresponding generalized eigenvectors are of the form $\mathbf{v}_{i\cdot p}^\mathbf{S} = \mathbf{v}_{i\cdot p}^\mathbf{W}\otimes\mathbf{v}_1^\mathbf{W}$. For eigenvalues constructed with $\lambda_2^\mathbf{A}$ it holds that $\lambda_2^\mathbf{A}\lambda_{i\cdot p}<\lambda_1^\mathbf{A}\lambda_{i\cdot p}$. This lets us simplify the statement:
    \begin{equation}
    \begin{split}
    \left(\frac{q_k}{|q_k|}\right)^k\frac{\mathbf{V}(\frac{1}{q_k}\mathbf{J})^k(c_1\mathbf{e}_1 + \dots c_n\mathbf{e}_n)}{\|\mathbf{V}(\frac{1}{q_k}\mathbf{J})^k(c_1\mathbf{e}_1 + \dots c_n\mathbf{e}_n)\|} 
    &= \left(\frac{q_k}{|q_k|}\right)^k \frac{c_{1\cdot p}\mathbf{v}_{1\cdot p}^\mathbf{S}+\dots +c_{i\cdot p}\mathbf{v}_{i\cdot p}^\mathbf{S}}{\|c_{1\cdot p}\mathbf{v}_{1\cdot p}^\mathbf{S}+\dots +c_{i\cdot p}\mathbf{v}_{i\cdot p}^\mathbf{S}\|} + \mathbf{r}_k \\
    &= \left(\frac{q_k}{|q_k|}\right)^k \frac{b\mathbf{u}\otimes\mathbf{v_1}^\mathbf{A}}{\|b\mathbf{u}\otimes\mathbf{v_1}^\mathbf{A}\|} + \mathbf{r}_k
    \end{split}
    \end{equation}
    for $b=c_{1\cdot p}\cdot\dots\cdot c_{i\cdot p}$, $\mathbf{u} = \mathbf{v}_{1\cdot p}+\dots +\mathbf{v}_{i\cdot p}$, and $\lim_{k\to\infty}\|\mathbf{r}_k\| = 0$ which converges to $\frac{\mathbf{v}_1}{\|\mathbf{v}_1\|}$ iff $\lambda_1 > 0$. Setting $\beta_k = \left(\frac{q_k}{|q_k|}\right)^k \frac{1}{\|b\mathbf{u}\otimes\mathbf{v_1}^\mathbf{A}\|}$ leads to Proposition~\ref{prop:power_kron}.
\end{proof}


\end{document}